\documentclass[11pt]{article}
\usepackage{amsfonts}
\usepackage{mathrsfs}
\usepackage{amsmath,amssymb}
\usepackage{mathtools}

\DeclareMathOperator{\EX}{\mathbb{E}}

\usepackage[latin1]{inputenc}
\usepackage{amsmath,amssymb}

\usepackage{amsthm}
\usepackage{latexsym}
\usepackage{epstopdf}
\usepackage{geometry}  
\usepackage{pict2e,picture}
\usepackage{bigints}  
\usepackage[thinlines]{easytable}
\usepackage{enumerate}
\usepackage{array}
\usepackage{fix-cm}
\usepackage{tikz}
\usetikzlibrary{matrix}

\newcolumntype{C}[1]{>{\centering\arraybackslash}m{#1}}
\newcolumntype{R}[1]{>{\raggedleft\arraybackslash}m{#1}}
\usepackage{lipsum}

\usepackage[symbol]{footmisc}

\usepackage[active]{srcltx}
\usepackage{graphicx}
\usepackage{epstopdf}
\usepackage{enumitem}   
\usepackage[T1]{fontenc}
\usepackage{amsmath}

\usepackage{footnote}
\usepackage{footmisc}
\makesavenoteenv{tabular}
\makesavenoteenv{table}

\usepackage{bbm}

\DeclareMathOperator*{\argmin}{arg\,min}

\usepackage[
bookmarks=true,         
bookmarksnumbered=true, 
colorlinks=true, pdfstartview=FitV, linkcolor=blue, citecolor=blue,
urlcolor=blue]{hyperref}

\newtheorem {theorem}{Theorem}[section]

\newtheorem {corollary}{Corollary}[section]

\newtheorem{lemma}{Lemma}[section]

\renewcommand\footnotemark{}

\usepackage{dsfont}
\date{\vspace{-5ex}}

\begin{document}

	\title{Shallow neural network representation of polynomials}
	
	\maketitle
\begin{center}

 Aleksandr Beknazaryan \footnote{abeknazaryan@yahoo.com}

\end{center}
	
	\begin{abstract}
			We show that $d$-variate polynomials of degree $R$ can be represented on $[0,1]^d$ as shallow neural networks of width $2(R+d)^d$. Also, by SNN representation of localized Taylor polynomials of univariate $C^\beta$-smooth functions, we derive for shallow networks the minimax optimal rate of convergence, up to a logarithmic factor, to unknown univariate regression function.
		\vskip.2cm

		\vskip.2cm \noindent {\bf Keywords}:
		\noindent  neural networks, polynomials, function approximation, nonparametric regression
		
		\vskip.2cm 
	\end{abstract}
\section{Introduction}
Polynomials play important role in a wide variety of scientific fields and, especially, in approximation theory. Due to the simplicity of their evaluation, differentiation and integration, polynomials are also widely used in practice. Theoretical properties of polynomials, and, in particular, their approximating capabilities, have been revealed by classical  results of Bernstein, Chebyshev, Mergelyan, Taylor, Weierstrass and of other famous mathematicians. Many of those results, together with interesting historical discussions about their derivations, are presented in the work \cite{S}. The established theory of polynomial approximations has also become a foundation for actively developing theory of approximations by neural networks. 

Shallow neural network approximations of polynomials in various norms and with various activation functions are given in the works \cite{A}, \cite{D1}, \cite{G}, \cite{M} and \cite{P}. In particular, using the property of Bernstein polynomials to simultaneously approximate functions and their derivatives, identical properties for shallow neural networks are derived in \cite{A}. Deep ReLU network approximations of localized Taylor polynomials of $C^\beta$-smooth functions on $[0,1]^d$ are constructed in \cite {SH} and \cite{Y}. Also, approximations of polynomial truncations of  multivariate Chebyshev series of analytic functions by deep networks is given in \cite{B}. Note that shallow and deep networks constructed in the above presented works approximate polynomials with given accuracy. Exact representations of polynomials (with no approximation error) by \textit{deep} RePU networks are constructed in the works \cite{CLM}, \cite{CM},  \cite{L1} and \cite{Mhaskar}, which, together with approximating properties of polynomials, allow to approximate general multivariate smooth functions with deep neural networks. 

In this work we first show that any polynomial of degree $R$ can be represented on $[0,1]^d$ as shallow neural network of width $2(R+d)^d$. In particular, this implies that the number of network parameters required for $\varepsilon$-approximation of $C^\beta$-smooth functions by shallow neural networks has order $O(\varepsilon^{-d/\beta})$. The representation is based on the non-singularity of generalized Vandermonde matrices and uses continuous, piecewise polynomial activation function. We note that similar techniques were also used for approximations constructed in the above presented works of Mhaskar and his coauthors and in the works \cite{CL} and \cite{MM}. Also, shallow network approximations with piecewise polynomial activation functions were given in \cite{Pet}. 

We also propose SNN representation of localized Taylor polynomials of univariate $C^\beta$-smooth functions. Following the standard procedure of bounding the entropies of approximating shallow networks and applying oracle inequality, we derive, up to a logarithmic factor $\log_2n$, the minimax optimal rate of prediction $n^{\frac{-2\beta}{2\beta+1}}$. Note that regression estimation with shallow networks was also considered in the works \cite{Bach}, \cite{Barron}, \cite{KB} and \cite{MG}. While for all $\beta>0$ the rate $n^{\frac{-2\beta}{2\beta+6}}$ derived for univariate model in \cite{MG} is slower than $n^{\frac{-2\beta}{2\beta+1}}\log_2n$ by a polynomial factor, the rates $n^{-1/2}\log_2n$ and $n^{-3/4}\log_2n$ derived (for very smooth functions) in \cite{Bach}, \cite{Barron} and \cite{KB}, respectively, are slower than $n^{\frac{-2\beta}{2\beta+1}}\log_2n$ given $\beta>3/2$. Recently, minimax rates of convergence for \textit{aggregated} one-hidden-layer networks were also derived in \cite{TD}.

\section{Multivariate polynomials}
Given an activation function $\rho:\mathbb{R}\to\mathbb{R}$, a shallow $\rho$-network on $[0,1]^d$ of width $P$ is a function of the form
$$\textbf{x}\mapsto\sum_{j=0}^{P-1}u_j\rho\bigg(v_j+\sum_{i=1}^dw_{ji}x_i\bigg)= U\rho(V^d_R+W\textbf{x}), \quad \textbf{x}=(x_1,...,x_d)\in[0,1]^d.$$ The weight matrix $W=\{w_{ji}\}\in\mathbb{R}^{P\times d}$ and the bias vector $V=(v_0,...,v_{P-1})\in\mathbb{R}^P$ constitute the hidden layer and the output layer is determined by the vector $U=(u_0,...,u_{P-1})\in\mathbb{R}^P$. We first consider shallow networks with activation function
\[\rho(x)=\begin{cases} 
      0, & x\leq 0 \\
      (x-k)^{k+1}+k, & k\leq x\leq k+1, k\in\mathbb{N}_0.
   \end{cases}\]
 Recall that for $\textbf{r}=(r_1, ...,r_d)\in\mathbb{N}_0^d$, the degree of $d$-variate monomial $\normalfont\textbf{x}\mapsto\textbf{x}^\textbf{r}:=x_1^{r_1}...x_d^{r_d}$ is the number $r=\sum_{i=1}^dr_i.$
\begin{lemma}\label{Monomial}
For $r\geq 2$ any linear combination of monomials of degree $r$ with coefficients bounded by $C$ can be represented on $[0,1]^d$ as shallow $\rho$-network of width $\binom{r+d-1}{d-1}+1$ with weights and biases bounded by $2C(r+d)^{d-1}(4d)^{2(r+d)^{2d-1}}$.
\end{lemma}
\begin{proof} Denote by $N:=\binom{r+d-1}{d-1}$ the number of $d$-dimensional monomials of degree $r$ and consider any linear combination $$\Sigma_{j=1}^Nc_j\textbf{x}^{\textbf{r}_j}$$ of monomials of degree $r$. For $\textbf{r}=(r_1, ...,r_d)\in\mathbb{N}_0^d$ with $\sum_{i=1}^dr_i=r,$ denote $a_{\textbf{r}}=\sum_{i=1}^dr_i(r+1)^{i-1}.$ Let $0=a_0<a_1<a_2<...<a_N$ be a set of $N+1$ pairwise distinct integers with $a_j:=a_{\textbf{r}_j}, j=1,...,N$. 
Then, by \cite{F}, Lemma 4, we have that for any pairwise distinct numbers $b_0,...,b_{N}\in(0,1/d],$ the matrix $[b_j^{a_i}]_{0\leq i,j\leq N}$ has positive determinant. Hence, there exist $u_0,...,u_N\in\mathbb{R}$ such that $\sum_{j=0}^{N}u_j=0$ and $\sum_{j=0}^{N}u_jb_j^{a_k}=c_kC_{\textbf{r}_k}$, where $C_{\textbf{r}_k}$ is the inverse of the coefficient of $\textbf{x}^{\textbf{r}_k}$ in the expansion of $(\sum_{i=1}^dx_i)^r$, $k=1,...,N$. As $b_j\in(0,1/d]$ then $\sum_{i=1}^dx_ib_j^{(r+1)^{i-1}}\in[0,1]$ for all $j=0,...,N,$ and for all $(x_1,...,x_d)\in[0,1]^d$. Hence,
$$\sum_{j=0}^{N}u_j\rho\bigg(r-1+\sum_{i=1}^dx_ib_j^{(r+1)^{i-1}}\bigg)=\sum_{j=0}^{N}u_j\bigg[\bigg(\sum_{i=1}^dx_ib_j^{(r+1)^{i-1}}\bigg)^r+r-1\bigg]=\Sigma_{j=1}^Nc_j\textbf{x}^{\textbf{r}_j}.$$
To bound the network weights we choose $b_j=(2d)^{-(j+1)},$ $j=0,...,N$. Denoting $e_j:=(2d)^{-a_j},$ using $a_j<(r+1)^d, j=0,...,N,$ and applying Theorem 1 from \cite{G1}, we get that the entries of the inverse matrix of $[b_j^{a_i}]_{0\leq i,j\leq N}$ are bounded by $$\frac{2^N}{\min\limits_{0\leq j\leq N}e_j}\max\limits_{0\leq j\leq N}\prod\limits_{\substack{i=0 \\ i\neq j}}^N\frac{1}{|e_j-e_i|}\leq (4d)^{(N+1)(r+1)^d}.$$ As $N\leq(r+d)^{d-1},$ then the weights $u_0,...,u_N$ are bounded by $C(N+1)(4d)^{(N+1)(r+1)^d}\leq 2C(r+d)^{d-1}(4d)^{2(r+d)^{2d-1}}$.
\end{proof}

\begin{theorem}\label{thm} Any polynomial of degree $R$ with coefficients bounded by $C$ can be represented on $[0,1]^d$ as shallow $\rho$-network of width $2(R+d)^d$ with weights and biases bounded by $$2C(R+d)^{d-1}(4d)^{2(R+d)^{2d-1}}.$$
\end{theorem}
\begin{proof} Monomials of degree $0$ and $1$ are, respectively, constants and linear functions $(x_1,...,x_d)\mapsto x_i$, $i=1,...,d$. Clearly, those functions can be represented as shallow $\rho$-networks of width $1$. Hence, by Lemma \ref{Monomial}, a polynomial of degree $R$ can be represented as a shallow $\rho$-network of width $d+1+\sum_{r=2}^R[\binom{r+d-1}{d-1}+1]\leq\sum_{r=0}^R[\binom{r+d-1}{d-1}+1]\leq2(R+d)^d.$
\end{proof}
For a domain $D\subset\mathbb{R}^d$ and for $\beta\in\mathbb{R}_+$ let \begin{align*}
\mathcal{C}^\beta_d(D, K):=\bigg\{f:D\to\mathbb{R}: \sum\limits_{0\leq|\boldsymbol{\gamma}|<\beta}\|\partial^{\boldsymbol{\gamma}}f\|_{L^\infty(D)}+\sum\limits_{|\boldsymbol{\gamma}|=\lfloor\beta\rfloor}\sup\limits_{\substack{\textbf{x},\textbf{y}\in D \\ \textbf{x}\neq \textbf{y}}}\frac{|\partial^{\boldsymbol{\gamma}}f(\textbf{x})-\partial^{\boldsymbol{\gamma}}f(\textbf{y})|}{|\textbf{x}-\textbf{y}|_\infty^{\beta-\lfloor\beta\rfloor}}\leq K\bigg\}.
\end{align*} 
Applying Theorem 2 from \cite{BBL}, we get 
\begin{corollary}For an open neighbourhood $D$ of $[0,1]^d$ let f be a function of class $\mathcal{C}^\beta_d(D, K)$. Then there is a shallow $\rho$-network $g$ of width $2(n+d)^d$ such that $$\sup\limits_{x\in[0,1]^d}|f(x)-g(x)|\leq Cn^{-\beta},$$
where $C=C(\beta, d, K)$ is some constant.
\end{corollary}
Thus, $O(\varepsilon^{-d/\beta})$ network parameters are needed to achieve the rate of approximation $\varepsilon$. This dependence of the number of network parameters on the approximation error is identical to the ones obtained in \cite{SH} and \cite{Y} for deep sparse ReLU networks. However, as Theorem \ref{thm} shows, the range of the weights of approximating shallow $\rho$-networks may be much larger than the weights of approximating deep ReLU networks (for comparison, the weights of deep ReLU networks constructed in \cite{SH} are all bounded by $1$). In the following section we consider representations of localized Taylor polynomials of univariate functions as shallow networks with smaller weights. 

In conclusion of this part we note that since on $(-\infty, 1]$ the function $\rho$ coincides with the ReLU function $x\mapsto\max\{0,x\}$, then any ReLU network on $[0,1]^d$ can be represented as a $\rho$-network with same architecture and same sparsity (using the positive homogeneity of ReLU function we can make the weights in hidden layers sufficiently small to assure that in each hidden layer the activation function receives a vector with coordinates from $[-1,1]$).

\section{Univariate localized Taylor polynomials}
For $\beta, K\in\mathbb{R}_+$ let 
\begin{align*}
\mathcal{C}^\beta(K):=\bigg\{f:[0,1]\to\mathbb{R}: \sum\limits_{0\leq r<\beta}\|f^{(r)}\|_{\infty}+\sup\limits_{\substack{x,y\in[0,1] \\ x\neq y}}\frac{|f^{(\lfloor\beta\rfloor)}(x)-f^{(\lfloor\beta\rfloor)}(y)|}{|x-y|^{\beta-\lfloor\beta\rfloor}}\leq K\bigg\}
\end{align*} 
be the ball of $\beta$-H\"older continuous functions on $[0,1]$ of radius $K$. Here $\|\cdot\|_\infty$ denotes the supremum norm of functions on $[0,1]$. For $a\in[0,1]$ let
\begin{align*}
P_a^\beta f(x)=\sum\limits_{0\leq r<\beta}f^{(r)}(a)\frac{(x-a)^r}{r!}
\end{align*} 
be the local Taylor polynomial of $f\in\mathcal{C}^\beta(K)$ around $a$. For $M\in\mathbb{N}$ denote
$$P_M^\beta f(x)=\sum\limits_{\ell=0}^MP_{\ell/M}^\beta f(x)(1-M|x-\ell/M|)_+=\sum\limits_{\ell=0}^M\sum\limits_{0\leq r<\beta}f^{(r)}(\ell/M)\frac{(x-\ell/M)^r}{r!}(1-M|x-\ell/M|)_+,$$
where $(x)_+=\max\{0,x\}$ is the ReLU function. From  \cite{SH1}, Lemma B.1, it follows that for $f\in\mathcal{C}^\beta(K)$
\begin{equation}\label{a}
\|P_M^\beta f-f\|_{\infty}\leq KM^{-\beta}.
\end{equation}

Let us now represent the function $P_M^\beta f$ as a shallow neural network. For a non-negative, even integer $k\in2\mathbb{N}_0$ with $m_k(m_k+1)\leq k<(m_k+1)(m_k+2), \; m_k\in\mathbb{N}_0,$ denote  $$a_k:=\frac{k}{2}-\frac{m_k(m_k+1)}{2}\;\;\; \textrm{ and } \;\;\; b_k:=\frac{(m_k+1)(m_k+2)}{2}-\frac{k}{2}.$$ Define an activation function 
\begin{align*}
\sigma(x)=\begin{cases} 
      0, & x\leq -1 \\
      (x-k)^{a_k}(1-b_k|x-k|)_+, & k-1\leq x\leq k+1, k\in2\mathbb{N}_0.
   \end{cases}
\end{align*} 
The non-constant pieces of the function $\sigma$ are defined on intervals of length $2$ because the terms $x-\ell/M,$ $x\in[0,1], \ell=0,...,M,$ appearing in the expression of $P_M^\beta f(x)$ above, are contained in the interval $[-1,1]$. Note that as $b_k\geq1$ then $(1-b_k)_+=0$ for all $k\in2\mathbb{N}_0$. Therefore, $\sigma(x)=0$ for $x\in\mathbb{Z}$, and, in particular, $\sigma$ is continuous on $\mathbb{R}$. Moreover, for each $r\in\mathbb{N}_0, M\in\mathbb{N},$ there is a $k\in2\mathbb{N}_0$ with $k\leq(M+r)(M+r+1)-2<(M+r+1)^2$ such that $a_k=r$ and $b_k=M$. Hence, denoting $$\mathcal{N}_{\sigma}(P,R):=\bigg\{x\mapsto U\sigma(V+Wx)|\;\; U, V, W\in[-R, R]^P;\;\; x\in[0,1]\bigg\}$$ to be the set of shallow $\sigma$-networks on $[0,1]$ of width $P$ with weights and biases from $[-R,R]$, we get 
\begin{lemma}\label{l1}
For any $f\in\mathcal{C}^\beta(K)$ and for any $M\in\mathbb{N}$ there is a network $g\in\mathcal{N}_{\sigma}((M+1)\lceil\beta\rceil, K+(M+\lceil\beta\rceil+1)^2)$ with $g=P_M^\beta f$.
\end{lemma}
In order to apply the previous lemma to the problem of non-parametric regression estimation we also need to bound the entropy of the set $\mathcal{N}_{\sigma}(P,R)$. Recall that for a set of functions $\mathcal{F}$ from $[0,1]$ to $\mathbb{R}$ the $\delta$-covering number $\mathcal{N}(\delta,\mathcal{F},\|\cdot\|_\infty)$ of $\mathcal{F}$ is the minimal number $N\in\mathbb{N}$ such that there exist $f_1,...,f_N$  from $[0,1]$ to $\mathbb{R}$ with the property that for any $f\in\mathcal{F}$ there is some $j\in \{1,...,N\}$ such that $\|f-f_j\|_{\infty}\leq\delta.$ The number $\log_2\mathcal{N}(\delta,\mathcal{F},\|\cdot\|_\infty)$ is then called a $\delta$-entropy of the set  $\mathcal{F}$. To bound the entropy of $\mathcal{N}_{\sigma}(P,R)$ we will need the following lemma about the Lipschitz continuity of the activation function $\sigma$.
\begin{lemma}
On the interval $[-L,L], L\in2\mathbb{N}_0,$ the function $\sigma$ is Lipschitz continuous with Lipschitz constant bounded by $L+1$.
\end{lemma}
\begin{proof}
Note that for each $k\leq L, k\in2\mathbb{N}_0,$ on the interval $[k-1, k+1]$ the functions $(x-k)^{a_k}$ and $(1-b_k|x-k|)_+$ are bounded by 1 and are Lipschitz continuous with Lipschitz constants $a_k$ and $b_k$, respectively. Therefore, on $[k-1, k+1]$ the function $\sigma$ is Lipschitz continuous with Lipschitz constant $a_k+b_k=m_k+1\leq k+1.$ Hence, the Lipschitz constant of $\sigma$ on $[-L, L]$ is bounded by $L+1$. 
\end{proof}
\begin{lemma}\label{ent} For any $\delta\in(0,1]$,
$$\log_2\mathcal{N}(\delta,\mathcal{N}_{\sigma}(P,R),\|\cdot\|_\infty)\leq3P\log_2(16P(R+1)^3/\delta).$$
\end{lemma}
\begin{proof}
Take any $\varepsilon>0$ and consider the set $\mathcal{N}(\varepsilon)$ of networks from $\mathcal{N}_{\sigma}(P,R)$ with weights and biases being of the form $m\varepsilon$, $m\in[- R/\varepsilon, R/\varepsilon]\cap\mathbb{Z}.$ Clearly, the cardinality of the set $\mathcal{N}(\varepsilon)$ is bounded by $(2R/\varepsilon+1)^{3P}.$ Take any $g(x)=U\sigma(V+Wx)\in\mathcal{N}_{\sigma}(P,R)$ and let $g_1(x)=U_1\sigma(V_1+W_1x)\in\mathcal{N}(\varepsilon)$ be such that $|U-U_1|_\infty\leq\varepsilon, |V-V_1|_\infty\leq\varepsilon$ and $|W-W_1|_\infty\leq\varepsilon,$ where $|\cdot|_\infty$ denotes the $l_\infty$ norm of $P$-dimensional vectors. As $V, W, V_\varepsilon, W_\varepsilon\in[-R, R]^P,$ then $V+Wx\in[-2R, 2R]^P$ and $V_\varepsilon+W_\varepsilon x\in[-2R, 2R]^P$ for all $x\in[0,1]$. By previous lemma we have that on $[-2R, 2R]$ the function $\sigma$ is Lipschitz continuous with Lipschitz constant bounded by $2R+3$. Hence, denoting by $|\cdot|_1$ the $l_1$ norm of $P$-dimensional vectors, we get that for all $x\in[0,1]$
\begin{align*}
&|U\sigma(V+Wx)-U_\varepsilon\sigma(V_\varepsilon+W_\varepsilon x)|\leq\\
& |U\sigma(V+Wx)-U\sigma(V_\varepsilon+W_\varepsilon x)|+|U\sigma(V_\varepsilon+W_\varepsilon x)-U_\varepsilon\sigma(V_\varepsilon+W_\varepsilon x)|\leq\\
&|U|_1|\sigma(V+Wx)-\sigma(V_\varepsilon+W_\varepsilon x)|_\infty+|U-U_\varepsilon|_\infty|\sigma(V_\varepsilon+W_\varepsilon x)|_1\leq\\
&PR(2R+3)|(V+Wx)-(V_\varepsilon+W_\varepsilon x)|_\infty+\varepsilon P\leq 2\varepsilon PR(2R+3)+\varepsilon P\leq 4\varepsilon P(R+1)^2.
\end{align*}
Taking $\varepsilon=\delta/(4P(R+1)^2)$ we get that the set $\mathcal{N}(\delta/(4P(R+1)^2))$ of cardinality bounded by $(16P(R+1)^3/\delta)^{3P}$ forms a $\delta$-covering of $\mathcal{N}_{\sigma}(P,R)$.
\end{proof}
 Let now $f_0\in\mathcal{C}^\beta(K)$ be an unknown regression function to be recovered from $n$ observed iid pairs $(X_i, Y_i)$, $i=1,...,n,$ following a regression model 
$$Y_i=f_0(X_i)+\epsilon_i,$$
where the standard normal noise variables $\epsilon_i$ are assumed to be independent of $X_i$. Denoting 
\begin{equation}\label{bound}
\mathcal{N}_{\sigma}(P,R,K):=\{g\in\mathcal{N}_{\sigma}(P,R)\;\; | \;\; \|g\|_\infty\leq2K\},
\end{equation}
we get
\begin{theorem}\label{thm2}
Let $M_n:=\lceil n^{\frac{1}{2\beta+1}}\rceil,$ $P_n:=(M_n+1)\lceil\beta\rceil$ and $R_n:=K+(M_n+\lceil\beta\rceil+1)^2$. Then, the prediction error of empirical risk minimizer 
$$\hat{f}_n\in\argmin_{f\in{\mathcal{N}_{\sigma}(P_n, R_n, K)}}\sum_{i=1}^{n}(Y_i-f(X_i))^2$$
is bounded by 
$$\EX_{f_0}[(\hat{f}_n(X)-f_0(X))^2]\leq  cn^{\frac{-2\beta}{2\beta+1}}\log_2n,$$
where $c=c(\beta, K)$ is some constant.
\end{theorem}
 \begin{proof}
From \eqref{a} and Lemma \ref{l1}, we have 
\begin{equation}\label{inf}
\inf\limits_{f\in{\mathcal{N}_{\sigma}(P_n, R_n, K)}}\EX[(f(X)-f_0(X))^2]\leq Kn^{\frac{-2\beta}{2\beta+1}}.
\end{equation}
Also, from Lemma \ref{ent} we have that 
\begin{equation}\label{ent1}
\log_2\mathcal{N}(\delta,\mathcal{N}_{\sigma}(P_n, R_n, K),\|\cdot\|_\infty)\leq Cn^{\frac{1}{2\beta+1}}\log_2(n/\delta),
\end{equation}
for some constant $C=C(\beta, K)$.

From \cite{SH}, Lemma 4, it follows that for any $\delta\in(0,1]$ 
\begin{equation}\label{oracle}
\begin{split}
&\EX_{f_0}[(\hat{f}_n(X)-f_0(X))^2]\leq\\
& 4\bigg[\inf\limits_{f\in{\mathcal{N}_{\sigma}(P_n, R_n, K)}}\EX[(f(X)-f_0(X))^2]+4K^2\frac{18\log_2\mathcal{N}(\delta,\mathcal{N}_{\sigma}(P_n, R_n),\|\cdot\|_\infty)+72}{n}+64\delta K\bigg].
\end{split}
\end{equation}
Thus, applying \eqref{inf} and \eqref{ent1} and taking $\delta=n^{\frac{-2\beta}{2\beta+1}}$ in  \eqref{oracle} proves the theorem.
 \end{proof}
 Note that as the definitions of the functions $P_M^\beta f$ and $\sigma$ suggest, in Lemma \ref{l1} we can consider only the networks for which the weight vector $W$ is the all-ones vector: $W=\textbf{1}:=(1,1,...,1)$. Also, there is a vector $V=V_{\beta,M}$ such that for all $x\in[0,1],$ the coordinates of $\sigma(V_{\beta,M}+\textbf{1}x)$ are the values $(x-\ell/M)^r(1-M|x-\ell/M|)_+,$ $0\leq\ell\leq M, 0\leq r<\beta.$ Hence, the coordinates of the output layer $U$, which becomes the only layer determined by $f$, are of the form $f^{(r)}(a)/r!\in[-K,K].$ As the function $\sigma$ is bounded by $1$ and for each $x\in[0,1]$ there are at most $2$ values of $\ell=0,...,M,$ for which $(1-M|x-\ell/M|)_+\neq 0,$ then the networks from the sets $\{U\sigma(V_{\beta,M}+\textbf{1}x);\; |U|_\infty\leq K\}$ are bounded by $2\lceil\beta\rceil K$. Hence, considering only the networks from those  sets, we can derive the same rate of convergence as in Theorem \ref{thm2}, and in this case the condition of boundedness of networks imposed in \eqref{bound} can be omitted.

\end{document}